\documentclass{article}
\usepackage{spconf,amsmath,graphicx}
\usepackage{hyperref}

\usepackage{graphicx,array,verbatim}
\usepackage{amsmath, amsthm, amssymb, amsfonts, cancel}
\usepackage{algorithm, algorithmic, ifsym, subfigure, bm}

\newtheorem{proposition}{Proposition}
\usepackage{epstopdf}
\newcolumntype{P}[1]{>{\centering\arraybackslash}p{#1}}

\DeclareMathOperator{\EE}{\mathbb{E}}

\title{Active Query-Driven Visual Search Using Probabilistic Bisection and Convolutional Neural Networks}
\twoauthors
 {Athanasios Tsiligkaridis} 
	{Boston University, ECE Dept.\\
	8 St. Mary's St., Boston, MA, USA\\
	atsili@bu.edu}
	{Theodoros Tsiligkaridis} 
	{MIT Lincoln Laboratory\\
	244 Wood St., Lexington, MA, USA\\
	ttsili@ll.mit.edu}
\begin{document}
%
\maketitle
\begin{abstract}
We present a novel efficient object detection and localization framework based on the probabilistic bisection algorithm.  A Convolutional Neural Network (CNN) is trained and used as a noisy oracle that provides answers to input query images.  The responses along with error probability estimates obtained from the CNN are used to update beliefs on the object location along each dimension.  We show that querying along each dimension achieves the same lower bound on localization error as the joint query design.  Finally, we compare our approach to the traditional sliding window technique on a real world face localization task and show speed improvements by at least an order of magnitude while maintaining accurate localization. \end{abstract} 
\vspace{-3mm}

\section{Introduction}
\label{sec:I}
\vspace{-3mm}
Object detection and localization has received much attention in the literature as it is of high importance in various applications such as security, surveillance, and tracking.  Current methods can be slow and sometimes computationally prohibitive for resource-constrained platforms, e.g., mobile devices, UAVs.  By guiding the localization process using efficient search methods, dramatic speed ups in computation can be achieved as we show in this work.  

 
Object localization is classically carried out using sliding window techniques, as discussed in \cite{Canevet:2014}, where a large set of image sections is formed by scanning an image; each section is then fed into an object classifier that provides a score on how likely it is to belong in a given subsection. Sliding window methods are widespread and well established but they require many classifier evaluations thus making such techniques computationally challenging.  

Feature-based classifiers represent image content using a feature set. Haar-wavelet features have been used for object localization tasks in \cite{Viola:2001, Viola:2004}.  Histogram of Oriented Gradients (HOG) features are popular for understanding shape and structure of content in an image and have been used in \cite{Dalal:2005} for human detection. Instead of using hand designed features, modern CNN methods perform feature learning given the raw pixels with enough training data and computing power. 

Object detection methods designed for a complete image characterization include regression methods and region-based methods. Regression methods, e.g., YOLO \cite{Redmon:2016}, use a single CNN to predict the object location and category but require annotated data which is expensive to accurately label and obtain. Region CNN (RCNN) methods \cite{Girshik:2014} consist of a region proposal stage \cite{Gu:2009, Uijlings:2013}, followed by a CNN to extract region features, a box regressor, and class-specific SVM classifiers for object labeling. Grid CNN \cite{Lu:2017} improves upon RCNNs by adding a grid CNN subnetwork to form translation-sensitive feature maps. We remark that the complexity of these methods is warranted for large-scale image analysis; but, in this work, we focus on efficient detection of a single object of interest in an image.

Recently, Bayesian methods have been proposed for efficient object localization. A multi-scale sliding window method followed by a non-adaptive query based Bayesian method was proposed in \cite{Rajan:2015} to locate faces. The authors in \cite{Sznitman:2010} detected faces using active testing with hierarchically partitioned query sets and multi-dimensional posterior updating. 

The Probabilistic Bisection Algorithm (PBA) is a target localization method that can be used to reduce uncertainty in searching for an object by sequentially asking questions on the existence of the object on slices of an image.  Since CNNs provide state-of-the-art feature learning and classification \cite{Qin:2018, Gu:2018}, both methods can be combined to provide a flexible way of learning features of objects and carrying out fast search.  This approach can also be extended to other applications, e.g., object search in video. 

The PBA was developed in \cite{Horstein:1963} and its optimality was proven in \cite{Jedynak:2012}.  When a single agent searches for a target by querying a noisy oracle, the PBA obtains an exponential convergence rate as shown in \cite{Waeber:2013}.  The PBA was extended to multiple targets in \cite{Rajan:2015} and to multiple agents in a centralized \cite{Tsiligkaridis:2014} and network \cite{Tsiligkaridis:2015, Tsiligkaridis:2016} setting.  The preceding works showed that all agents in a network converge to the correct consensus through local information sharing. Further, \cite{Tsiligkaridis:2017} considered a distributed version of the PBA and showed that all agents' beliefs concentrate at the correct location exponentially fast.  This work focuses on the single-agent setting where the PBA is used to guide a trained CNN classifier to localize a target in an image. 


The contributions of this paper are as follows. First, we present a novel framework for object localization that combines the search speed of PBA and the feature learning capability and computational efficiency of CNNs for vision problems. The CNN serves as an oracle that provides noisy responses to input image queries. Second, we present a practical method for estimating the error probabilities corresponding to these queries which are used for Bayesian updating of the belief along each dimension. Third, we show that independent querying along each dimension is equivalent to joint querying in terms of a lower bound on the mean square error. Finally, we apply our technique to both a synthetic and a real world face localization task and show that we maintain accuracy and achieve a dramatic increase in computational efficiency over the traditional sliding window method. 

\vspace{-1mm}
\section{Algorithm Development} 
\label{sec:AD}
\vspace{-3mm}
We consider a target space $\mathcal{X}$ modeled by an input image $\bm{A}$ of size $N_1 \times N_2$, along with two probability distributions $p_{1,t}$ and $p_{2,t}$, one for each image dimension.  Each distribution represents the probability that a center of an object exists at a specific index bin.  At each iteration, each distribution is bisected and a query region is formed and fed into a CNN-based classifier.  The CNN returns a noisy response derived from its softmax layer and is used to update the posterior distributions. Due to the single object assumption inherent in the PBA, we seek to locate only a single object in an image; as future work, $n$-ary partitioning and querying may be used to locate multiple objects. 

It is also possible to use a joint query policy that maintains a joint distribution over $\mathcal{X}$.  This approach presents difficulties since the query set is non-trivial to form and updating the posterior is computationally expensive.  In contrast, our independent representation is beneficial since it allows for simple query region formation and updating.

Proposition \ref{prop:equiv} shows that independent querying on the probability distributions of each dimension is equivalent to the joint query policy in terms of mean square error. 
  
\begin{proposition} \label{prop:equiv}
Consider the independent query policy $\pi_I$ with equal allocation of queries per dimension, and the joint query policy $\pi_J$.  Both policies achieve the same lower bound on the localization error given by:
\begin{equation}
	K d e^{-2nC(\epsilon)/d} \leq \EE[\parallel X_n- X^*\parallel_2^2]
\end{equation} 
where $K$ is a constant, $d$ is the target space dimension, $n$ is the number of queries, $C(\epsilon)$ is the binary symmetric channel capacity with error probability $\epsilon$ \cite{Cover:2006} where all queries are assumed to have an error probability $\epsilon \in [0,\frac{1}{2})$, $X^*$ is the target location, and $X_n$ is the median estimate.
\end{proposition}
\begin{proof}
Theorem 4 from \cite{Tsiligkaridis:2014} yields the lower bound on the joint policy given by $	K d e^{-2nC(\epsilon)/d} \leq \EE[\parallel X_n^{\pi_J}- X^*\parallel_2^2]
$. Applying this theorem again independently in each dimension, we have $ K e^{-2n_i C(\epsilon)} \leq \EE[(X_{n,i}^{\pi_I}-X_i^*)^2]$. Using $n_i \approx n/d$, we further obtain:
\begin{align*}
	\EE&[\parallel X_n^{\pi_I}- X^*\parallel_2^2] = \sum_{i=1}^d \EE[(X_{n,i}^{\pi_I}-X_i^*)^2] \\
							       &\geq \sum_{i=1}^d K e^{-2n_i C(\epsilon)} \approx K d e^{-2nC(\epsilon)/d} 
\end{align*}
\end{proof}

Our algorithm will use a CNN as a noisy oracle that will accept query regions as input.  These query regions will be sections of the input image $\bm{A}$ based on the query point at a specific iteration $t$.  For example, if we consider the column ($x$) direction, we first bisect the posterior $p_{2,t}$ and obtain the query point $X_{2,t}$.  With this, two query regions $S_{1,t} = \{ (i,j) | i \in [1,N_1], j \in [1,X_{2,t}] \}$ and $S_{2,t} = \{ (i,j) | i \in [1,N_1], j \in [X_{2,t}+1,N_2] \}$ are formed and one is randomly selected with probability $\frac{1}{2}$.  Next, we want to see if there is an object of interest in this region to accordingly update the posterior along the $x$ direction.  

We note that our CNN only accepts inputs of a standard user-defined size which we will define as being a square. The input query region can be long or narrow based on the initial image dimensions and the query point. Directly resizing this region into a square and using it as input can be problematic since the resizing operation can distort content in the section and can cause incorrect CNN responses.  To avoid this, we partition the query region into $Q_t$ square blocks based on its minimum dimension.  We then resize each block to the set input size; resizing already square blocks will not cause content distortion.  Each block $B_{i,t}, i = 1,...,Q_t$, is fed into the CNN and we obtain a set of responses $\mathcal{Y}_t = \{y_{i,t}|y_{i,t} \in \{ 0,1 \}, i = 1,...,Q_t \}$ and a matrix of softmax outputs $\bm{V}_t$ where each row is the CNN's softmax output for block $B_{i,t}$: $v_{i,t} = [p_{\{ \text{object},i,t \}},p_{\{ \text{no object},i,t \}}]$.  The final decision is obtained as: $y_t = \bigcup_{y \in \mathcal{Y}_t} y$.

The matrix $\bm{V}_t$ and the response set $\mathcal{Y}_t$ are used to update the posterior.  If the CNN provides high confidence on object existence in the section, we want to drastically update the posterior; but, if the CNN provides low confidence, then the distribution should not be changed heavily.  For this, we will define a parameter $\epsilon$ that parameterizes the likelihood. 

Let $T_F = \{ i | p_{\{ \text{object},i,t \}} > p_{\{ \text{no object},i,t \}} \}$ be the index set of blocks from the query region with the object of interest.  For the $i$th block that is formed, the error probability is estimated as $\epsilon_i = 1 - \text{max} \{ p_{\{ \text{object},i,t \}},p_{\{ \text{no object},i,t \}} \}$ where $\epsilon_i \in [0,\frac{1}{2})$.  The total error probability $\epsilon$ is obtained as:

\begin{equation} 
\label{eq:epsilon}
  \epsilon =
  \begin{cases}
                                   		\min_{i \in T_F} \epsilon_i & \text{if $T_F \neq \varnothing$} \\
                                   		\min_{i} \epsilon_i & \text{if $T_F = \varnothing$} 
  \end{cases}
\end{equation}

The likelihood function for the PBA is:
\begin{equation} \label{eq:li}
	f(u,y_t) = f_1 (y_t) I (u \in S_{k,t}) + f_0 (y_t) I (u \notin S_{k,t}) 
\end{equation}
where $u$ denotes the $x$ or $y$ index bin, $k \in \{ 1,2 \}$ represents either dimension, and $f_1(y_t)=(1-\epsilon)^{I(y_t=1)} \epsilon^{I(y_t=0)}, f_0(y_t)=1-f_1(y_t)$. Equation \ref{eq:li} is used in updating the posteriors. Over many iterations, the posterior distributions concentrate at the coordinates of the center of an object of interest.

Algorithm \ref{alg:alg1} summarizes our complete object localization method.

\begin{algorithm}
\caption{ Hybrid PBA + CNN Localization Algorithm }
\label{alg:alg1}
\begin{algorithmic}[1]
\STATE \textbf{Input:}  {Image $\bm{A}$ of size $N_1 \times N_2$}
\STATE \textbf{Output:} {$(X_{1}^*,X_{2}^*) =$ object center coordinates}

    \STATE Initialize: $p_{k,0} [u] = \frac{1}{N_i} \forall u \in \{1,...,N_k \}, k \in \{ 1,2 \}$.
   
    \REPEAT \STATE For $k \in \{1,2 \}$; $t =$ current iteration, Do: \\
        				\STATE (a) Bisect $p_{k,t} [u]$ to obtain the query point $X_{k,t}$ s.t. $\sum_{u=1}^{X_{k,t}-1} p_{k,t}[u]  \leq \frac{1}{2}, \sum_{u=1}^{X_{k,t}} p_{k,t}[u]  > \frac{1}{2}$\\
				and randomly select a query region with equal probability: $S_{k,1}$ or $S_{k,2}$.\\
			  \STATE (b) Break up query region into $Q_t$ square blocks, input each block $B_{i,t}$ into CNN, and obtain response $y_t = \cup_{y \in \mathcal{Y}_t} y$ \\
				\STATE (c) Obtain $\epsilon$ from Equation \ref{eq:epsilon} and update distributions: \\$p_{k,t+1}[u] = p_{k,t}[u] \cdot 2 f(u,y_t)$
				
    \UNTIL {convergence}
\end{algorithmic}
\end{algorithm}

\vspace{-2mm}
\section{Experimental Results}
\label{sec:ER}
\vspace{-3mm}

\subsection{Synthetic Dataset - Star In Noise Localization}

We first apply Algorithm 1 to a synthetic example of localizing a star in the presence of salt and pepper noise.  For the oracle, we train a small CNN with two convolution and max pooling layers using $1000$ training and $800$ testing examples, all of which had the same noise level as that depicted in the rightmost image in Figure \ref{fig:star_test_images}.  The test results in Figure \ref{fig:star_test_images} show that our proposed method achieves identical localization performance as the sliding window method.

\begin{figure}[!ht]
	\includegraphics[width=0.4\textwidth, height=0.135\textwidth]{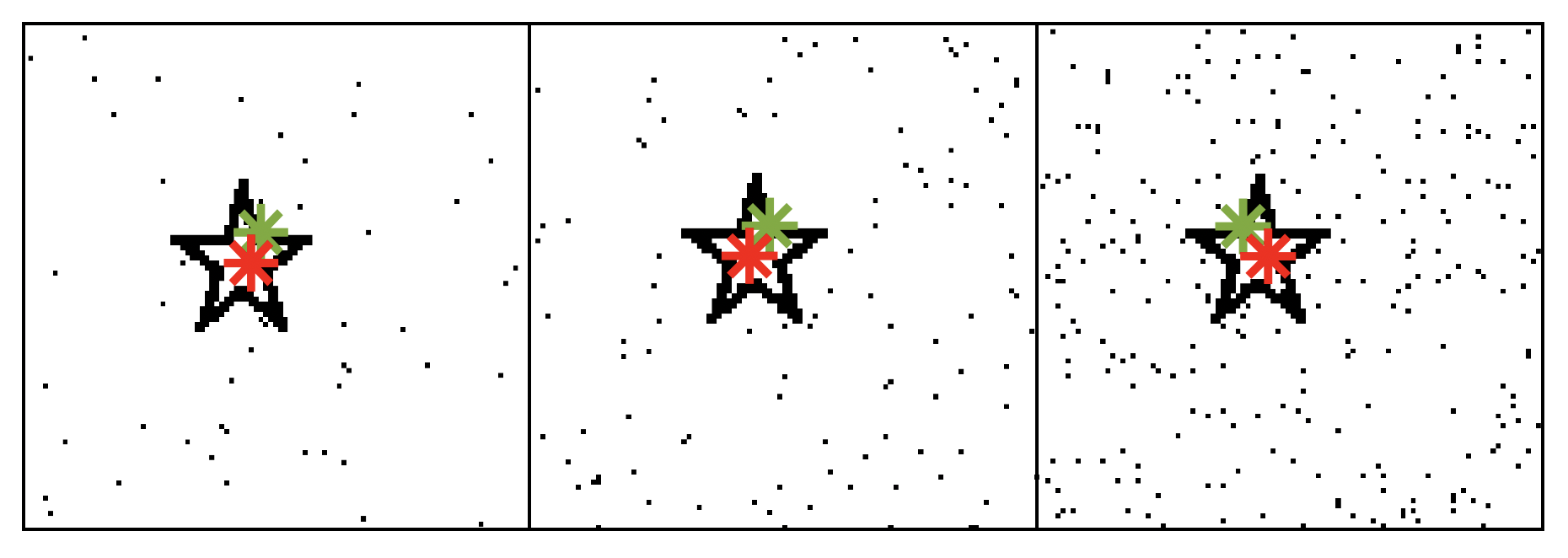}
	\centering
	\vspace{-4mm}
	\caption{\small Object localization using sliding window approach (red) and our proposed Hybrid PBA + CNN method (green). The markers represent the centers obtained from both respective methods.  Our method accurately finds the star target.}
	\label{fig:star_test_images}
\end{figure}

\subsection{Real Dataset - Celebrity Face Localization}

We now want to apply Algorithm 1 to a face localization task.  We first train a face identifying CNN as our oracle using the CelebA dataset developed in \cite{Liu:2015}.  This dataset contains two types of celebrity images: one with both face and background (e.g, see entire image in Figure \ref{fig:forming_data}), and the other solely with face (e.g., see columns $2$ and $4$ in Figure \ref{fig:training_data}). We do not directly train the classifier using the images with face and background since the background must not be dominant and should be minimized to allow for emphasis of the face content. Thus, we partition all face and background images based on the minimum dimensions of each image as displayed in Figure \ref{fig:forming_data}. For this process, we first find the minimum dimension $d_{min}^{1}$ of the entire image and form a square of size $d_{min}^{1} \times d_{min}^{1}$ that contains a small amount of background and a face.  Then, we find the minimum dimension $d_{min}^{2}$ of the remaining portion and form a square of size $d_{min}^{2} \times d_{min}^{2}$ that contains solely background.  With this, we obtain a set of faces (with minor background) and backgrounds that we use along with the images with just faces to form a training set.   Example training images are displayed in Figure \ref{fig:training_data} where columns $1$ and $3$ contain images with face and background, columns $2$ and $4$ contain face images, and columns $5$ and $6$ contain background images such as wallpapers, shoulders, hands, etc. In addition, we also split the face images in columns $2$ and $4$ and form partial face images that we also include in the training set.  This is done to allow the CNN to understand cases where the query region could possibly contain a fraction of a face. Ultimately, we have $8000$ training and $2000$ testing images of size $100 \times 100$ which are used to form our CNN; a three layer architecture is used consisting of convolution and max pooling layers. 

\begin{figure}[!ht]
	\includegraphics[width=0.285\textwidth, height=0.275\textwidth]{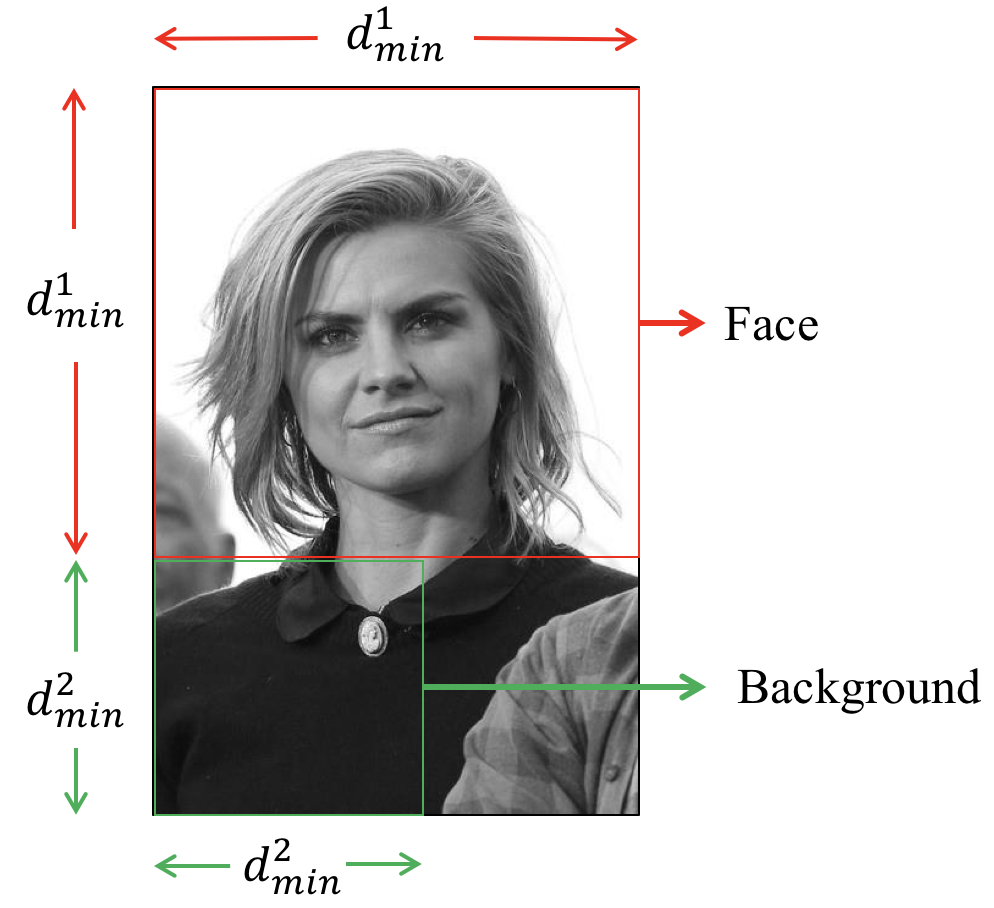}
	\centering
	\vspace{-3mm}
	\caption{\small Methodology for forming face and background images from the CelebA images. The presented image is partitioned into an image with more face and less background and an image with only background.}
	\label{fig:forming_data}
\end{figure}

\begin{figure}[!ht]
	\includegraphics[width=0.465\textwidth, height=0.42\textwidth]{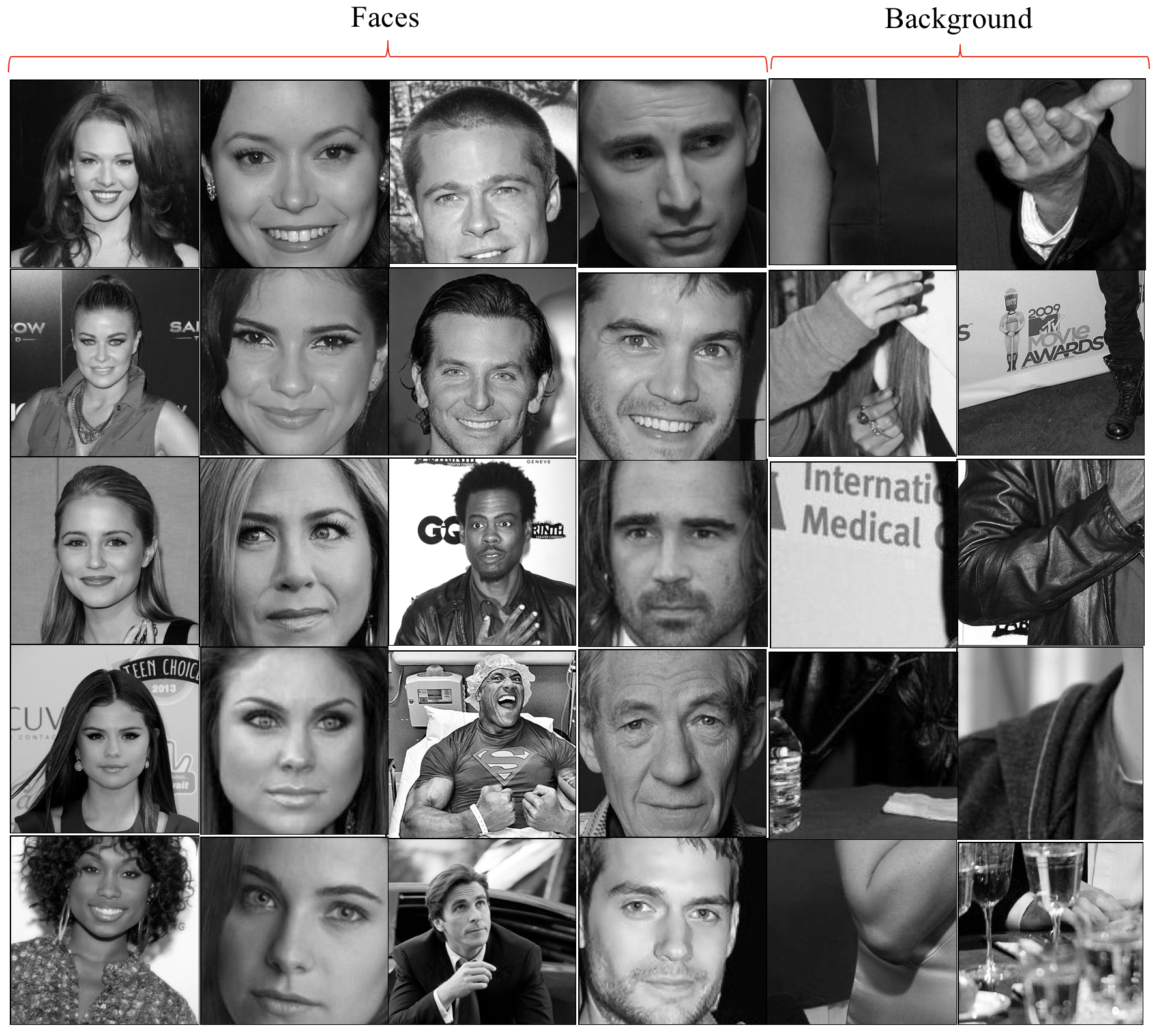}
	\centering
	\vspace{-5mm}
	\caption{\small Example training images used for training our CNN.}
	\label{fig:training_data}
\end{figure}

Next, we test Algorithm 1 on a set of celebrity images (not from the CelebA dataset) and estimate face center coordinates.  After we obtain the center coordinates, we form a bounding box by using our trained CNN with increasingly larger boxes around the obtained center.  Our stopping condition is when the probability of a face obtained from the softmax layer exceeds $0.90$. Figure \ref{fig:test_images} shows our test results where the green and red bounding boxes are obtained from our proposed method and the sliding window approach, respectively.  The images are labeled as (A) - (F) and in all cases we obtain correct face centers and bounding boxes, even though the faces are of different sizes and positions.  


\begin{figure}[!ht]
	\includegraphics[width=0.5\textwidth, height=0.575\textwidth]{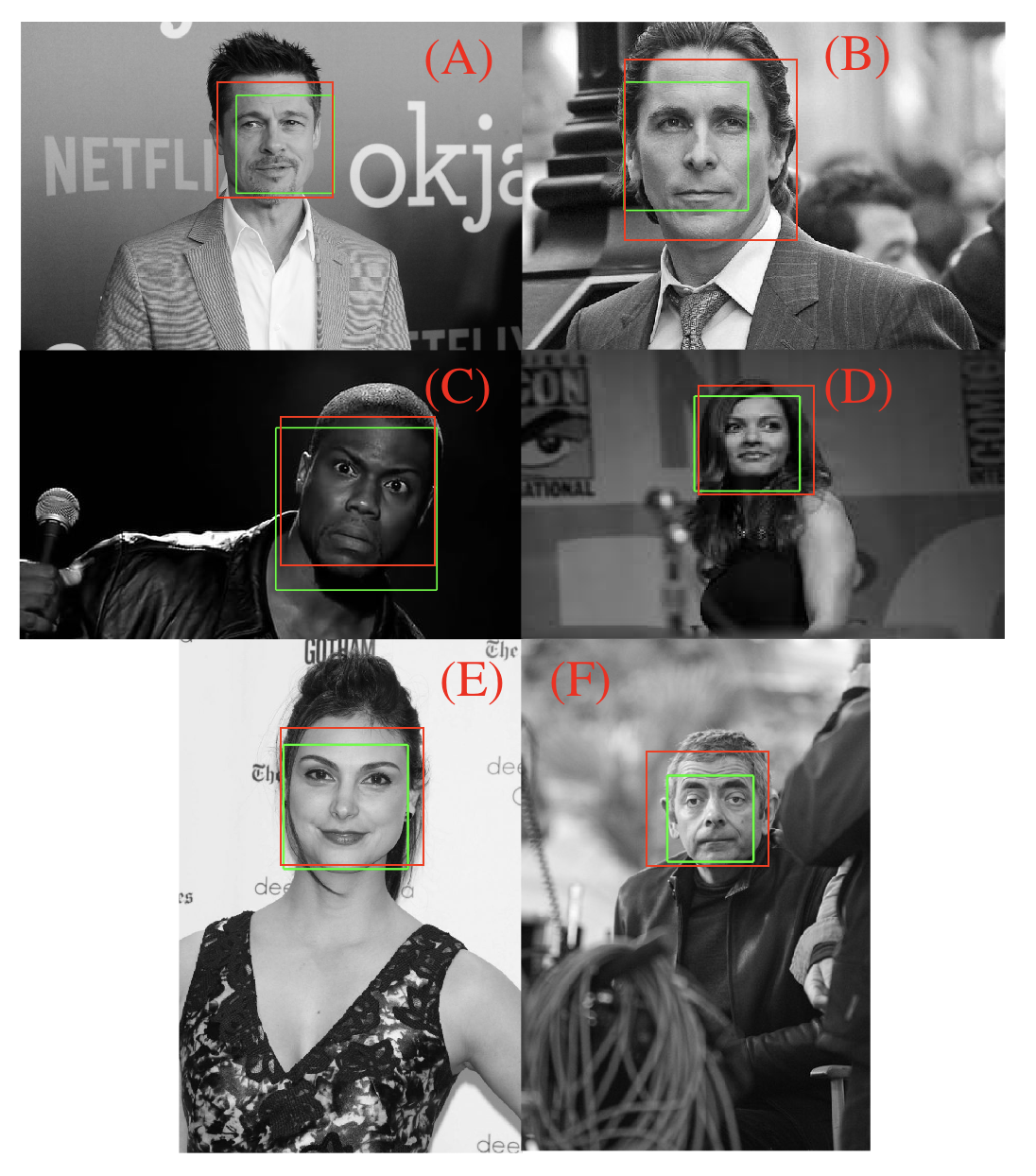}
	\centering
	\vspace{-8mm}
	\caption{\small Sample test images and detected faces using the sliding window method (red) and our proposed Hybrid PBA + CNN Localization Algorithm (green).}
	\label{fig:test_images}
\end{figure}

Our method allows for fast and efficient localization as compared to traditional sliding window techniques.  Table \ref{table:cnn} lists the amount of CNN calls needed to acquire the center face coordinates for images (A) - (F); these values are obtained by averaging CNN call counts over $20$ trials per image.  For a comparison, we consider the number of CNN calls that are required when using a sliding window technique where an image is scanned using overlapping windows with a fixed window shift and various window sizes.  Since the images are of different sizes, different amounts of CNN calls will be required for a given window size and shift; on average, for a selected window size, our method is at least an order of magnitude faster than the sliding window approach.

\begin{table}[h!]

\centering
 \begin{tabular}{||P{0.75cm}| P{0.85cm}| P{1.6cm}| P{1.6cm}| P{1.6cm}||} 
 \hline
 {\small Image Label} & {\small Hybrid PBA + CNN} & {\small Sliding Window: $w = 100$px, $s = 25$px} & {\small Sliding Window: $w = 150$px, $s = 25$px}  & {\small Sliding Window: $w = 200$px, $s = 25$px}  \\ [0.5ex] 
 \hline\hline
 A & 41 & 912 & 792 & 680 \\ 
 \hline
 B & 34 & 375 & 299 & 231  \\
 \hline
 C & 34 & 703 & 595 & 495  \\
 \hline
 D & 26 & 540 & 448 & 364   \\
 \hline
 E & 36 & 828 & 714 & 608  \\   
 \hline
 F & 46 & 260 & 198 & 144  \\ [1ex] 
 \hline
\end{tabular}

\caption{\small Comparison of the classifier call counts for the Hybrid PBA + CNN method along with the traditional sliding window method for a fixed window shift and various window sizes. Our method is 43$\times$ faster, on average over all the selected window sizes and tested images, according to the number of CNN calls.}
\label{table:cnn}
\end{table}

\vspace{-3mm}
\section{Conclusion}
\vspace{-3mm}

In this paper, we proposed a new and efficient algorithm for object localization based on combining the probabilistic bisection algorithm with convolutional neural network classification. In addition, we showed that working and updating beliefs along each direction in comparison to the more complex joint case is equivalent in terms of a lower bound on the mean square error.  We also presented experiments on a face localization task and showed significant speed improvements over traditional methods. Future work may extend this work to localize multiple objects in images.



\clearpage
\bibliographystyle{IEEEbib}
\bibliography{paper_v2}

\end{document}